\def\year{2021}\relax
\documentclass[letterpaper]{article}
\usepackage{cfMGML}
\usepackage{times}
\usepackage{helvet}
\usepackage{courier}
\usepackage[hyphens]{url}
\usepackage{graphicx}
\usepackage{natbib} 
\usepackage{caption}
\newtheorem{theorem}{Theorem}
\newtheorem{proof}{Proof}

\usepackage{amsmath}
\usepackage{booktabs}
\usepackage{algorithm}
\usepackage{algorithmic}
\usepackage{multirow}
\usepackage{subfigure}
\usepackage{verbatimbox}
\usepackage[switch]{lineno}

\title{Towards Coarse and Fine-grained Multi-Graph Multi-Label Learning }

\author{
    Yejiang Wang\textsuperscript{\rm 1}, Yuhai Zhao\textsuperscript{\rm 1},\\
    Zhengkui Wang\textsuperscript{\rm 2},  \\
    Chengqi Zhang\textsuperscript{\rm 3}
    \\
}
\affiliations{

    \textsuperscript{\rm 1}Northeastern University\\
    \textsuperscript{\rm 2}Singapore Institute of Technology\\
    \textsuperscript{\rm 3}University of Technology Sydney\\

}

\begin{document}


\maketitle

\begin{abstract}
    Multi-graph multi-label learning (\textsc{Mgml}) is a supervised learning framework, which aims to learn a multi-label classifier from a set of labeled bags each containing a number of graphs. Prior techniques on the \textsc{Mgml} are developed based on transfering graphs into instances and focus on learning the unseen labels only at the bag level. In this paper, we propose a \textit{coarse} and \textit{fine-grained} Multi-graph Multi-label (cfMGML) learning framework which directly builds the learning model over the graphs and empowers the label prediction at both the \textit{coarse} (aka. bag) level and \textit{fine-grained} (aka. graph in each bag) level. In particular, given a set of labeled multi-graph bags, we design the scoring functions at both graph and bag levels to model the relevance between the label and data using specific graph kernels. Meanwhile, we propose a thresholding rank-loss objective function to rank the labels for the graphs and bags and minimize the hamming-loss simultaneously at one-step, which aims to addresses the error accumulation issue in traditional rank-loss algorithms. To tackle the non-convex optimization problem, we further develop an effective sub-gradient descent algorithm to handle high-dimensional space computation required in cfMGML. Experiments over various real-world datasets demonstrate cfMGML achieves superior performance than the state-of-arts algorithms.
\end{abstract}

\section{Introduction}
The importance of structure information in machine learning has received increasing attention, the application of which ranges from bioinformatics and chemistry to social networks. Traditional Multi-instance Learning (\textsc{Mil}) models each object of interest (e.g. image or text) into a bag-of-instances representation \cite{dietterich1997misl}. For example, in the image classification, each image (as shown in Fig.\ref{fig1a}) is represented as a bag and samples (aka. regions) inside the image can be represented as instances (as shown in Fig.\ref{fig1b}). Although \textsc{Mil} has been used in many different applications, it faces the challenges of capturing the dependency of the samples in one object. In reality, many real-world objects are inherently complicated and the samples in the object may have dependency with each other, which has unfortunately been discarded in the bag-of-instances representation in \textsc{Mil}. The dependency or the relationship plays an important role in describing the object. As shown in Fig.\ref{fig1b}, the vectors can only show the pixels of images without the adjacency relations between pixels. 

Much recent research effort has been devoted to adopt a better structured data (graph) to represent the object in the learning. Existing works can be mainly categorized into below two main classes: 1.) One class of learning methods focuses on representing one object into a single graph (as shown in Fig.\ref{fig1c}) \cite{zhou2009multi-instance-non-iid,2010-graph-mi-6}. A single graph representaion can capture the global relationships of the object, where each sample is represented as one node with a vector value and the sample relationship is captured as the edges. However, the vectors representation for each node can only show the visual features and lacks of the local structure spatial relations in the sample. 2.) Another recent class of learning methods adopts a bag-of-graphs representation of the given object, which uses a graph to represent each sample in the object in the learning \cite{wu2014boosting,pang2018semi}. Fig.\ref{fig1d} shows such a bag-of-graphs representation of Fig.\ref{fig1a}, where the image is divided into multiple samples and each sample is represented by a graph that captures the adjacency relations between pixels. This is also known as Multi-graph Learning (\textsc{Mgl}), which has drawn increasing interest in machine learning community. For example, Multi-graph Single-label (\textsc{Mgsl}) learning is proposed to design learning models to train and predict the data by assuming there is only a single label on the object. Later on, Multi-graph Multi-label (\textsc{Mgml}) learning is proposed to release the constraint of single label issue and to enable the calssification model to describe each bag of graphs with multiple labels. The second class of \textsc{Mgl} learning is proven to be more efficient than the \textsc{Mil} and the first class methods \cite{zhu2018multi}.       

However, the disadvantage of prior \textsc{Mgl} approaches is two-fold. First, although the input data of existing \textsc{Mgl} works is bag-of-graphs, they still need to explicitly transform each graph into binary feature instances and then design the models based on these instances. This graph-to-vector transformation, unfortunately, may fail to capture all the relevant graph structure information that is of value for the learning. Second, all the existing \textsc{Mgl} techniques are only capable of predicting the class labels for the \textit{coarse-grained} level (to label the graph bags), but not for the \textit{fine-grained} level (to label the graphs inside each bag). Automated labeling in the \textit{fine-grained} level has been proven to be of great importance in many applications, such as image or text annotation \cite{briggs2012rank}. Annotating large dataset require significant manual inspection effort by domain experts, which is time-consuming and expensive.


In this paper, we propose a new \textit{coarse} and \textit{fine-grained} multi-graph multi-label (cfMGML) learning framework, which is the first attempt to develop \textsc{Mgml} solution directly over graphs inside the bag and predicting labels at both graph and bag levels simultaneously. Different to prior \textsc{Mgl} methods, we first map each graph within a bag into a high-dimensional feature space with a specific graph kernel that retains the relevant information of graph for classification instead of transfering graphs to binary instances. Then, we define kernel-based scoring functions for each class to label each graph while we adopt another bag-level scoring functions which iteratively select the most valuable graphs to represent and label the bag. Moreover, we present a new thresholding rank loss objective function, which also inherently alleviates a cumulative error issue. Furthermore, to solve the non-convex optimization problem, we develop an effective subgradient descent algorithm with a formal proof of the bound. We have further introduced several real-world labeled datasets for cfMGML learning, which include both images and text. These labeled datasets will be publicly available to benefit the community. Comparative experiments over various datasets and baseline algorithms clearly validate the effectiveness of cfMGML.

\begin{figure}[tb]
    \centering
    \vspace{-0.0em}
    \subfigure[\emph{Origin image}]{\label{fig1a}\includegraphics[scale=0.47]{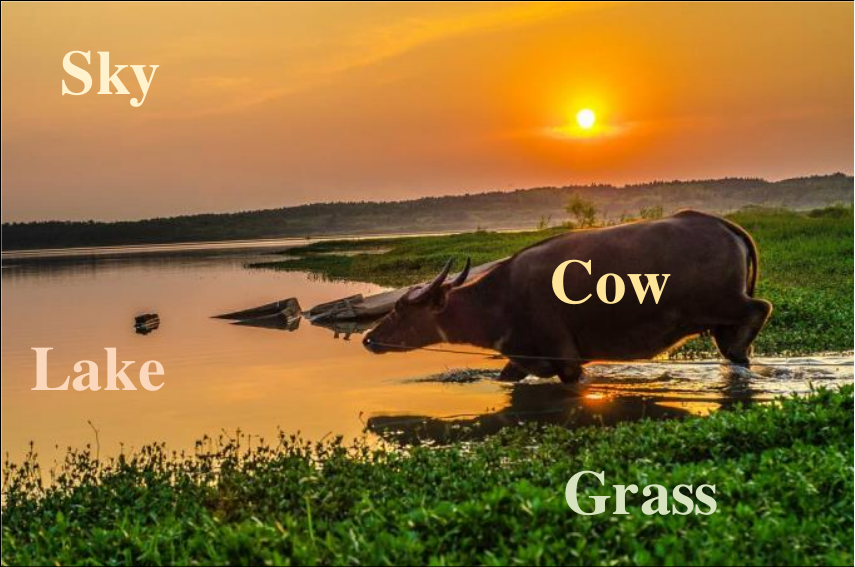}}
    \subfigure[\emph{A bag-of-instances}]{\label{fig1b}\includegraphics[scale=0.47]{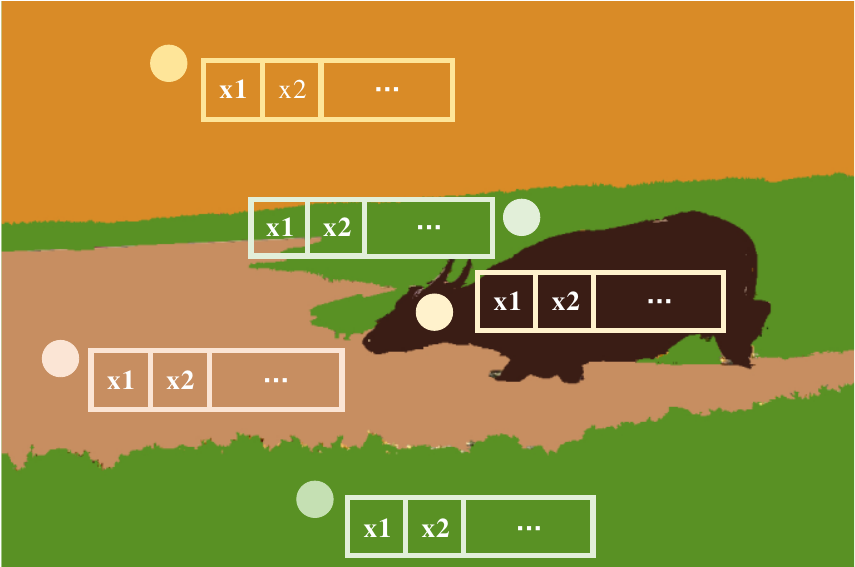}}
    \subfigure[\emph{A single graph}]{\label{fig1c}\includegraphics[scale=0.47]{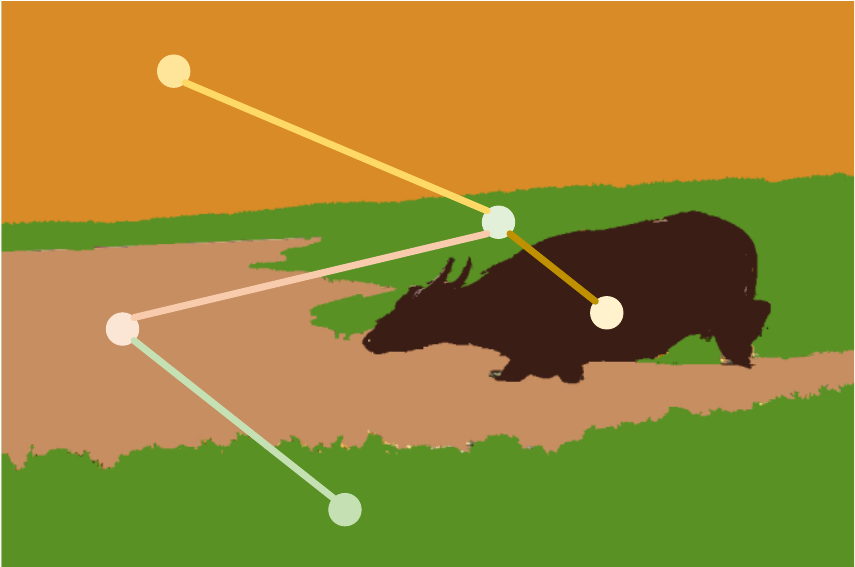}}
    \subfigure[\emph{A bag-of-graphs}]{\label{fig1d}\includegraphics[scale=0.47]{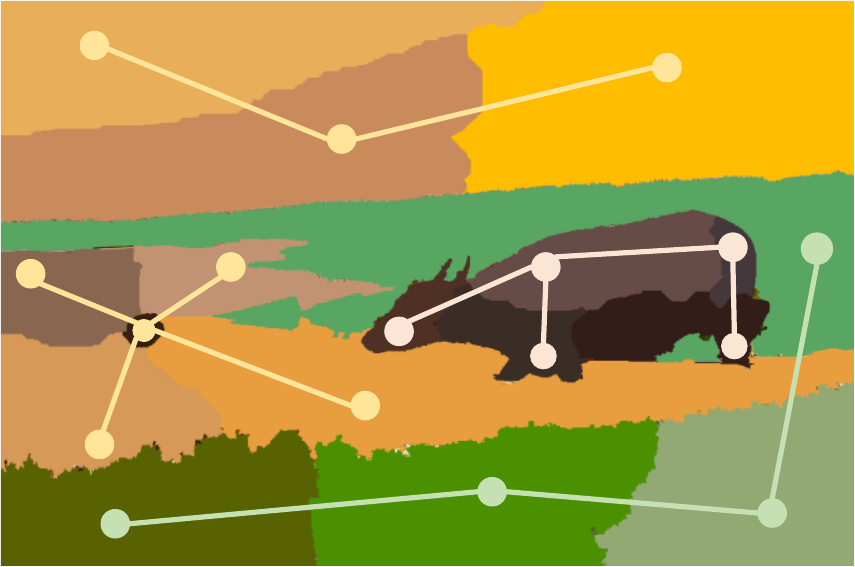}}
    \vspace{-1em}
     \caption{An image and its bag with several representaions.}
    \vspace{-1em}
     \label{fig1}
\end{figure}



The remainder of the paper is organized as follows. Section 2 introduces our proposed cfMGML learning framework with details. In Section 3, we present the experimental evaluation details. Section 4 provides the a review of related works, and we conclude the paper in Section 5. 

\section{The Proposed Solution}

 In this section, we first introduce the problem definition and an overview of cfMGML, followed by the detailed design.

\subsection{Problem Definition}
In the cfMGML learning, we are given a set of $n$ multi-graph training bags $\mathcal{Z}$$=$$ \{\mathcal B_1,\mathcal B_2,...,\mathcal B_n\}$, where $\mathcal B_i$$=$$\{g_{i1},g_{i2},...,g_{in_i}\}$ is a multi-graph bag used to represent the $i^{th}$ object (e.g. an image or text), where $g_{ij}$ denotes the graph in $\mathcal B_i$ and $n_i$ denotes the number of graphs in $\mathcal B_i$. Each bag $\mathcal B_i$ is labeled with a label set $Y_i^+ \subseteq \mathcal{Y}$ and let $Y_i^-$ denotes the complementary of $Y_i^+$, where $\mathcal{Y} $$=$$ \{1,2,...,C\}$ and $C$ is the total number of classes. Moreover, $g_{ij}$ is represented as $g=(\mathcal{V,\it{E},\mathcal S, l)}$, where $\mathcal{V}$ denotes a set of vertices; $\it{E}\subseteq \mathcal{V \times V}$ is a set of edges; $\mathcal S$ is a label set of vertices inside the graph; and $l:\mathcal V \rightarrow \mathcal S$ is a mapping that assigns real-valued vector (node-attributes graph) or real number (node-label graph) to the vertices. Note that $\mathcal{Y}$ and $\mathcal S$ are different. The elements in $\mathcal{Y}$ represent class labels of graphs or bags (e.g. cow or sky), while the elements in $\mathcal S$ represent the property of vertices (e.g. the ordinal number of the vertices in graph $g_{ij}$). In this paper, the proposed algorithm can be used in the case of each graph with multiple labels.

The traditional \textsc{Mgml} learning only needs to predict the label set for each of $m$ testing bags $\{\hat{\mathcal B}_1,\hat{\mathcal B}_2,...,\hat{\mathcal B}_m\}$ of graphs. Differently, the aim of cfMGML is to train one supervised machine learning model such that, for any given unlabeled testing example represented as a multi-graph bag $\hat{\mathcal B}_i$$=$$\{g_{i1},g_{i2},...,g_{im_i}\}$, where $m_i$ denotes the number of graph in $\hat{B}_i$, we can predict the labels of $\hat{B}_i$ and the label of each graph $g_{ij}$ inside $\hat{\mathcal B}_i$.



\subsection{Overview of cfMGML}


%
Graph differs from instances in that it contains not only the global information but also local spatial information. Therefore, given a set of labeled multi-graph bags, we first map each graph of a bag into a high-dimensional feature space with a specific kernel to retain the relevant graph information. Then, we define graph-level scoring functions to score each graph for each label, which can model the relevance between the label and the graph. Moreover, in order to reduce the computational complexity, especially when there are a large number of training examples, we describe a criterion to select the most valuable graph from bags as the representative graphs to define bag-level scoring functions. Then we present a thresholding rank-loss objective to take label relations into account, by assuming that the relevant labels is expected to be ranked before irrelevant ones for each bag, and tackle a cumulative error issue by introducing a virtual zero label for thresholding. At last, in consideration of the extremely high computational complexity of the Frank-Wolfe method which is used to solve quadratic programming problem, we design a subgradient descent algorithm to optimize the loss function and its convergence is proved.

\subsection{Graph Labeling}
Existing \textsc{Mgml} methods try to use instance vector to represent each graph in the training. However, this will likely lose the structural information in the learning. To tackle this issue and label a graph, we define kernel-based graph-level scoring functions for each class label $c$ as follows
\begin{equation}
    f_c(g)=\langle w_c,\phi(g) \rangle
\end{equation}
where $w_c$ is the weight and $g$ is the training graph from training bags. $\phi(\cdot)$ maps the graph space into a Hilbert space with inner product $\langle \cdot,\cdot \rangle$. The scoring functions enable cfMGML to model the relevance between the label and the graph, where a higher value of the scoring function indicates a higher relevance. Note that the proposed algorithm can be used in the case of each graph with multiple labels by setting a threshold for the scoring functions. In addition, cfMGML is general enough to adopt any kernel functions to different classification problems. For example, we use Graph-Hopper (GH) kernel \cite{feragen2013GraphHopper} for node-attributes graphs which can be used to represent image data and Weisfeiler-Lehman (WL) kernel \cite{shervashidze2011weisfeiler} is applied for the node-label graphs which is used to represent text data.


\subsection{Bag Labeling}

Consider that, in the training datasets, there are only labels for the bag instead of the graph of the bag. We need to construct a bag-level scoring function for bag labeling.  
Intuitively, the labels of bag are the union of graph labels. In other words, there exists at least one graph in bag $\mathcal B$ from class $c$ if $c$ is a label of $\mathcal B$. With such an assumption, we define the bag-level scoring function at each class as
\begin{equation}
    F_c(\mathcal B)=\max_{g_j \in \mathcal{B}} f_c(g_j)=\langle w_c,\phi(\overline{g}_c) \rangle
\end{equation}
where $\overline{g}_c$ is the representative graph which achieves the maximum value for each bag (i.e. $\overline{g}_c$$=$$\arg\max_{g_{j}\in \mathcal B}\langle w_c,\phi(g_{j})\rangle$) on class $c$. Specifically, if a graph in a bag is associated with a label, the label set of the bag must contain that label.

\subsection{Thresholding Rank-Loss Objective Function}
To model the dependencies between predicted graph labels and the ranking of scores for each label, we minimize the pairwise approximate rank-loss which imposes a penalty on a classifier for the incorrect ranking. Traditional rank-loss methods try to find the ideal threshold value based on the learned parameter value after the training step. However, it faces a cumulative error issue since the training step has the ranking loss error bias and the threshold value also has the error bias \cite{2012-Xu12a}. In this section we tackle this issue by directly introducing a virtual zero label for thresholding in the training process. It can be formulated as
\begin{equation}
    \resizebox{.91\linewidth}{!}{$
    \displaystyle
    \begin{aligned} 
    \mathcal{L}(\mathcal B_i) &= \frac{1}{|Y_i^+|^2}\sum\limits_{p \in Y_i^+}[[F_p(\mathcal B_i) \leq 0]] + \frac{1}{|Y_i^-|^2}\sum\limits_{q \in Y_i^-}[[F_q(\mathcal B_i) \geq 0]]\\
    &+ \frac{1}{|Y_i^+||Y_i^-|}\sum\limits_{p \in Y_i^+}\sum\limits_{q \in Y_i^-}[[F_p(\mathcal B_i) \leq F_q(\mathcal B_i)]]
    \end{aligned}
$}
\label{cfmgml:objective}
\end{equation}
where $|\cdot|$ represents cardinality, and $[[\cdot]]$ is the indicator function. The first two terms in the formula correspond to the thresholding step, which is equivalent to the minimization of the hamming loss. The last term aims to encourage the score for the positive label in $Y_i^+$.

Due to discontinuousness and non-convexity, we instead explore hinge loss which is a larger margin loss and one of optimal convex surrogate losses \cite{rosasco2004hingeloss}. Then, the learning is defined as:
\begin{equation}
    \resizebox{.91\linewidth}{!}{$
    \displaystyle
    \begin{aligned} 
        \mathop{{ min}}\limits_{W} \quad & \frac{\lambda}{2}\sum_{c=1}^{C} \left\|w_c\right\|^2 + \frac{1}{n} \sum_{i=1}^{n} \left\{ \frac{1}{|Y_i^+|^2} \sum_{p \in Y_i^+} \xi_{p,0}^i + \frac{1}{|Y_i^-|^2} \sum_{q \in Y_i^-} \xi_{0,q}^i \right.\\
        & \left. +\frac{1}{|Y_i^+||Y_i^-|} \sum_{p \in Y_i^+} \sum_{q \in Y_i^-} \xi_{p,q}^i \right\} \\
        \textrm{s.t. } \quad & F_p(\mathcal B_i) - 0 \geq 1-\xi_{p,0}^i, \quad p \in Y_i^+ \\
        \quad & 0 - F_q(\mathcal B_i) \geq 1-\xi_{0,q}^i, \quad q \in Y_i^- \\
        \quad & F_p(\mathcal B_i)-F_q(\mathcal B_i) \geq 2-\xi_{p,q}^i, \quad (p,q) \in Y_i^+ \times  Y_i^- \\
        & \xi_{p,q}^i,\xi_{p,0}^i,\xi_{0,q}^i \ge 0, \quad i=1,...,n
    \end{aligned}
$}
\label{cfmgml:constrained}
\end{equation}
Other variant approaches of rank-loss method \cite{2008-Calibrated-Rank-SVM,2014-rank-svmz} introduce another zero label parameter (i.e., $w_0$), and thus increase the complexity of the model (i.e., the hypothesis set). Differently, we introduce a zero label while avoiding adding another parameters $w_0$ and reduce the complexity. By transforming Eq.(\ref{cfmgml:constrained}) into unconstrained optimization, the surrogate loss can be expressed as:
\begin{equation}
    \resizebox{.91\linewidth}{!}{$
    \displaystyle
    \begin{split}
        \mathop{{ min}}\limits_{W} 
        \quad & \frac{\lambda}{2}\sum_{c=1}^{C} \left\|w_c\right\|^2 + \frac{1}{n}\sum_{i=1}^{n} \left\{ \frac{1}{|Y_i^+|^2} \sum_{p \in Y_i^+} |1-\langle w_p,\phi(\overline{g}_{i,p})\rangle|_+ \right.\\
        & + \frac{1}{|Y_i^-|^2} \sum_{q \in Y_i^-} |1+\langle w_q,\phi(\overline{g}_{i,q})\rangle|_+ + \frac{1}{|Y_i^+||Y_i^-|} \\
        & \left. \times \sum_{p \in Y_i^+} \sum_{q \in Y_i^-} |2+\langle w_q,\phi(\overline{g}_{i,q})\rangle-\langle w_p,\phi(\overline{g}_{i,p})\rangle|_+ \right\}
    \end{split}
    \label{cfmgml:unconstrained}
$}
\end{equation}
where $|a|_+=a$ if $a>0$, otherwise $|a|_+=0$. The first term of objective controls model complexity by penalizing the norm of weight matrix $W$. The terms in brace measure the difference of label hyperplanes $f_{c}(g)$ on graph $g$ between positive label and zero one, zero label and negative one, positive label and negative one, respectively.

\subsection{The Algorithm}

Unfortunately, the regularized surrogate loss Eq.(\ref{cfmgml:unconstrained}) is non-convex because graph $\overline{g}$ in last term is not an independent variable. To solve this problem,  we first find the representative graph $\overline{g}$ and all the representative graphs are treated as constant. Then we design a new subgradient descent algorithm to optimize loss function. The subgradient w.r.t. $w_c$ of the objective at iteration $t$ is computed as:
\begin{equation}
    \resizebox{.91\linewidth}{!}{$
    \displaystyle
    \begin{aligned}
    \nabla_t^c=&\lambda w_t^c + \frac{1}{n} \sum_{i=1}^n \left\{ \frac{1}{|Y_i^+|^2} \sum_{p \in Y_i^+} \alpha_{p}^c \phi(\overline{g}_{i,c}) [[1-\langle w_t^p, \phi(\overline{g}_{i,p}) \rangle \geq 0]] \right.\\
    & + \frac{1}{|Y_i^-|^2} \sum_{q \in Y_i^-} \alpha_{q}^c \phi(\overline{g}_{i,c}) [[1+\langle w_t^q, \phi(\overline{g}_{i,q}) \rangle \geq 0]] + \frac{1}{|Y_i^+||Y_i^-|}\\
    & \left. \times \sum_{p \in Y_i^+} \sum_{q \in Y_i^-} \beta_{p,q}^c \phi(\overline{g}_{i,c}) [[2+\langle w_t^q, \phi(\overline{g}_{i,q}) \rangle \geq \langle w_t^p, \phi(\overline{g}_{i,p}) \rangle]] \right\}
    \end{aligned}
$}
\end{equation}
where $\alpha_{x}^c$$=1$ $if$ $x$$==$$c;$ $0$ $otherwise$, and $\beta_{p,q}^c$$=1$ $if$ $c$$==$$q;$ $-1$ $else$ $if$ $c$$==$$q;$ $0$ $otherwise$.
We then update $w_{t+1}^c$$=$$w_t^c-\eta_t \nabla_t^c$, where $\eta_t$$=$$1/(\lambda t)$ is the step size. Obviously, by eliminating recursion the update $w_{t+1}^c$ can be rewritten as:
\begin{align}
    \resizebox{.91\linewidth}{!}{$
    \displaystyle
    \begin{aligned}
    w_{t+1}^c = & \frac{1}{\lambda t} \sum_{u=1}^t \frac{1}{n} \sum_{i=1}^n \left\{ \frac{1}{|Y_i^+|^2} \sum_{p \in Y_i^+} \alpha_{p}^c \phi(\overline{g}_{i,c}) [[1-\langle w_t^p, \phi(\overline{g}_{i,p}) \rangle \geq 0]] \right.\\
    & + \frac{1}{|Y_i^-|^2} \sum_{q \in Y_i^-} \alpha_{q}^c \phi(\overline{g}_{i,c}) [[1+\langle w_t^q, \phi(\overline{g}_{i,q}) \rangle \geq 0]] + \frac{1}{|Y_i^+||Y_i^-|}\\
    & \left. \times \sum_{p \in Y_i^+} \sum_{q \in Y_i^-} \beta_{p,q}^c \phi(\overline{g}_{i,c}) [[2+\langle w_t^q, \phi(\overline{g}_{i,q}) \rangle \geq \langle w_t^p, \phi(\overline{g}_{i,p}) \rangle]] \right\}
    \end{aligned}
    \label{cfmgml:omega}
$}
\end{align}
By interchanging the order of summation, we define
\begin{equation}
    \begin{cases}
        \gamma_{t+1}^{p,i} = \sum_{u=1}^t [[1-\langle w_t^p, \phi(\overline{g}_{i,p}) \rangle \geq 0]]\\
        \delta_{t+1}^{q,i} = \sum_{u=1}^t [[1+\langle w_t^q, \phi(\overline{g}_{i,q}) \rangle \geq 0]]\\
        \zeta_{t+1}^{p,q,i} = \sum_{u=1}^t [[2 + \langle w_u^q, \phi(\overline{g}_{i,q}) \rangle \geq \langle w_u^p, \phi(\overline{g}_{i,p}) \rangle]]
    \end{cases}
    \label{cfmgml:sum}
\end{equation}

\noindent where $p,q,i$ are fixed in each iteration. The summation formula (\ref{cfmgml:sum}) can be transformed as an iterative formula:
\begin{equation}
    \resizebox{.91\linewidth}{!}{$
    \displaystyle
    \begin{cases}
        \gamma_{t+1}^{p,i} = \gamma_{t}^{p,i} + [[1-\langle w_t^p, \phi(\overline{g}_{i,p}) \rangle \geq 0]] \\
        \delta_{t+1}^{q,i} = \delta_{t}^{q,i} + [[1+\langle w_t^q, \phi(\overline{g}_{i,q}) \rangle \geq 0]] \\
        \zeta_{t+1}^{p,q,i} = \zeta_{t}^{p,q,i} + [[2 + \langle w_t^q, \phi(\overline{g}_{i,q}) \rangle \geq \langle w_t^p, \phi(\overline{g}_{i,p}) \rangle]]
    \end{cases}
    \label{cfmgml:condition}
$}
\end{equation}
For convenience, we  suppose $\kappa_{t+1}^{c,i} $$=$$ |Y_i^-|^2 \sum_{p \in Y_i^+} \alpha_{p}^c \gamma_{t+1}^{p,i}$, $\nu_{t+1}^{c,i} = |Y_i^+|^2 \sum_{q \in Y_i^-} \alpha_{q}^c \delta_{t+1}^{q,i}$, $\mu_{t+1}^{c,i} = |Y_i^+||Y_i^-| \sum_{p \in Y_i^+}  $ \\
$\sum_{q \in Y_i^-} \beta_{p,q}^c \zeta_{t+1}^{p,q,i}$ and $z_i$$=$$\lambda t n |Y_i^+|^2|Y_i^-|^2$,
Eq.(\ref{cfmgml:omega}) can be rewritten as
\begin{equation}
    w_{t+1}^c = \sum_{i=1}^n \frac{1}{z_i} (\kappa_{t+1}^{c,i} + \nu_{t+1}^{c,i} + \mu_{t+1}^{c,i}) \times \phi(\overline{g}_{i,c})
\end{equation}
At next iteration $t+2$, with $w_{t+1}^c$, compute $\langle w_{t+1}^c, \phi(\overline{g}_{i,c}) \rangle$ of the condition in (\ref{cfmgml:condition}) as:
\begin{equation}
    \langle w_{t+1}^c, \phi(\overline{g}_{i,c}) \rangle = \sum_{h=1}^n \frac{1}{z_h} (\kappa_{t+1}^{c,h} + \nu_{t+1}^{c,h} + \mu_{t+1}^{c,h}) K(\overline{g}_{h,c}, \overline{g}_{i,c})
\end{equation}
where $K(\overline{g}_{h,c}, \overline{g}_{i,c})$$=$$\langle \phi(\overline{g}_{h,c}), \phi(\overline{g}_{i,c}) \rangle$. Note that the proposed algorithm does not need to directly access to the high-dimensional feature space, $\phi(\overline{g}_{i,c})$ and the weight $w_{t+1}^c$, which only provides kernel evaluations.

The pseudo code of cfMGML is presented in Algorithm \ref{alg:algorithm}. In the training stage, cfMGML repeats for $R$ rounds. There are two phases at each round $r$. The first phase (line 1,15) updates the representative graphs for the training bags using the last weights from subgradient descent. Note that we randomly select a graph from a bag as representative graph at the first round. The second phase (line 3-13) runs the subgradient descent algorithm a total of $T$ iterations using the representative graphs. At the first iteration ($t$$=$$1$), we initialize the variables $\boldsymbol \kappa, \boldsymbol \nu, \boldsymbol \mu$, where $\boldsymbol \kappa $$=$$\{\kappa_{c,i}\}$, $\boldsymbol \nu $$=$$\{\nu_{c,i}\}$ and $\boldsymbol \mu $$=$$\{\mu_{c,i}\}$, to zero and update it using the final weights from previous round. Otherwise we use the updated $\boldsymbol s$, where $\boldsymbol s $$=$$ \boldsymbol \kappa$$+$$\boldsymbol \nu$$+$$\boldsymbol \mu$, from the earlier iteration to perform the subgradient descent step. Then the variable $\boldsymbol s$ is projected to the sphere space where the optimal solution exist. Note that the norm of $w_{t+1}^c$ in line 13 is given by $\left\|w_{t+1}^c \right\|^2 $$=$$ \sum_{i=1}^n (1/z_i) s_{t+1}^{c,i} K(\overline{g}_{i,c},\overline{g}_{i,c})$. Therefore, it does not need to directly access the high-dimensional feature space.

{\renewcommand\baselinestretch{1}\selectfont
\renewcommand{\algorithmicrequire}{\textbf{Input:}}   
\renewcommand{\algorithmicensure}{\textbf{Output:}}
\renewcommand{\algorithmicrepeat}{\textbf{Repeat:}}   
\renewcommand{\algorithmicuntil}{\textbf{Until:}}
\renewcommand{\algorithmicreturn}{\textbf{Return:}}

\begin{algorithm}[tb]
    \caption{cfMGML}
    \label{alg:algorithm}
    \begin{algorithmic}[1] 

    \REQUIRE train data $\{\mathcal B_i,Y_i^+\}_{i=1}^n,\lambda$
    \ENSURE $\boldsymbol{s,\overline{g}}$

    \STATE Let $\boldsymbol{\kappa,\nu,\mu} \gets 0$ and randomly initialize $\boldsymbol{\overline{g}}$
    \STATE {\bfseries for $r=1:R$}
    \STATE \quad {\bfseries for $t=1...T$}
    \STATE \qquad {\bfseries for $i=1:n,p \in Y_i^+,q \in Y_i^-$}
    \STATE \quad \qquad {\bfseries if $F_{p}(\mathcal B_i) \leq 1$}
    \STATE \qquad \qquad $\kappa_{p,i} \gets \kappa_{p,i} + 1$
    \STATE \quad \qquad {\bfseries if $F_{q}(\mathcal B_i) \geq -1$}
    \STATE \qquad \qquad $\nu{q,i} \gets \nu_{q,i} + 1$
    \STATE \quad \qquad {\bfseries if $F_{q}(\mathcal B_i) \geq F_p(\mathcal B_i)-2$}
    \STATE \qquad \qquad $\mu_{p,i} \gets \mu_{p,i} + 1$
    \STATE \qquad \qquad $\mu_{q,i} \gets \mu_{q,i} - 1$
    \STATE \qquad $\boldsymbol{s} \gets \boldsymbol{\kappa}+\boldsymbol{\nu}+\boldsymbol{\mu}$
    \STATE \qquad $\boldsymbol{s} \gets \min\{1,\sqrt{\frac{2}{\lambda}} / \sum_{c=1}^C \left\|w_{t+1}^c \right\|^2\} \cdot \boldsymbol{s}$
    \STATE \quad {\bfseries for $i=1...n,c=1...C$}
    \STATE \qquad $\overline{g}_{i,c} \gets \arg\max_{g\in \mathcal B_i}\langle \boldsymbol{s}_{c,:},K(\overline{g}_{:,c},g) \rangle$
    \STATE \textbf{return} $\boldsymbol{s,\overline{g}}$

    \end{algorithmic}
\end{algorithm}
}

\noindent\textbf{Convergence Analysis.} Theoretically, the proposed iterative approach can converge by the following theorem.
\begin{theorem}
    Given the training set $\{(\mathcal B_i,Y_i^+): i$$=$$1...n\}$ and let $K$$=$$\max_{i,j \in \{1,...,n\}:g_i \in \mathcal B_i, g_j \in \mathcal B_j} K(g_i,g_j)$, then the sub-gradient descent step achieves a solution within $\sigma$ of optimal via $\mathcal{O}(\frac{K^2}{\lambda \sigma})$ iterations with $\mathcal{O}(\frac{n^2 K^2 R C}{\lambda \sigma})$ time complexity.
\end{theorem}

\begin{proof}
    The asymptotic convergence rate for the subgradient descent method with learning rate $\eta_t$$=$$1/(\lambda t)$ given in \cite{shalev2007pegasos} can be expressed as,
    \begin{equation*}
        \min_t \mathcal{L}(W_t) \leq \mathcal{L}(W^*) + \frac{M(1+\ln(T))}{2\lambda T}
    \end{equation*}
    where $M$ denotes the upper bound of $\sum_{c=1}^C \left\| \nabla_t^c \right\|^2$.
    From Eq.(\ref{cfmgml:unconstrained}), we have
    \begin{equation}
        \frac{\lambda}{2}\sum_{c=1}^C \left\| w_c^* \right\|^2 \leq \mathcal{L}(W^*) \leq \mathcal{L}(W=0) = 1
    \end{equation}
    Then we compute the upper bound $M$ as follows,
    \begin{align}
        \begin{aligned}
        \sum_{c=1}^C \left\| \nabla_t^c \right\|^2 & \leq 2(\sum_{c=1}^C \lambda^2 \left\| w_t^c \right\|^2 + CK^2)\\
        & \leq 4 \lambda + 2CK^2 = M
        \end{aligned}
    \end{align}
    Moreover, the value of $\ln(T)$ can be treated as constant in practical situations since $T$ is sufficiently small. Hence we have $T$$=$$\mathcal{O}(\frac{K^2}{\lambda \sigma})$. Then the time complexity of proposed algorithm is $\mathcal{O}(\frac{n^2 K^2 R C}{\lambda \sigma})$.
\end{proof}

Due to $K$, $R$, $C$ and $\lambda$ are constants in training process, the run time of our algorithm is mainly dependent on the number of training examples $n$ and the solution accuracy $\sigma$ at the subgradient descent step. Compared with Frank-Wolfe method whose time complexity is $\mathcal{O}(\frac{n^4 D^2}{C \sigma})$, where $D$ represents the regularization constant \cite{elisseeff2002kernel-fw-method}, it can be observed that cfMGML needs much less time complexity than Frank-Wolfe method to achieve a solution within $\sigma$ of optimal.

\section{Experiments}

\subsection{Baseline Algorithms}
We compare the performance for the graph and bag levels separately with corresponding baseline algorithms in each level. For the graph labeling, we compare cfMGML with below baselines: 1.) Hinged hamming-loss cfMGML (\textsc{Hlk}): \textsc{Hlk} differs from cfMGML by using only the hinged hamming loss instead of ranking loss, which imposes a penalty on a classification \cite{elisseeff2002ml-hamming}. 2.) MIMLfast: MIMLfast differs from cfMGML in that it is a multi-instance multi-label algorithm proposed in \cite{huang2018fast}. It can predict both the instance and bag labels, and is chosen to study the performance gain between instance-based and graph-based multi-label learning. 3.) \textsc{Sgsl svm}: \textsc{Sgsl svm} is trained with labeled graphs instead of bags and predicts graph label using \textsc{Libsvm} \cite{chang2011libsvm} with a graph kernel. The aim of comparing this algorithm is to provide an accuracy upper bound for cfMGML. 4.) \textsc{Dummy} classifier (\textsc{Dummy}): This assigns the most common label in the train dataset to each graph in test dataset. This algorithm aims to provide the performance lower bound for cfMGML.

For the bag labeling, the following methods are compared: MIMLfast, \textsc{M3Miml} \cite{zhang2008m3miml}, \textsc{Hlk}, MGMLent \cite{zhu2018multi}, \textsc{Sgsl svm} and \textsc{Dummy}. Note that MIMLfast and \textsc{M3Miml} use the bag-of-instances representation for data instead of bag-of-graphs representation.



\begin{table*}[tb]
    \caption{ Predictive performance (mean $\pm$ standard deviation) on Graph Label prediction of each comparing methods on the real world datasets.}
    \label{result2}
    \begin{scriptsize}
    \begin{center}
        \small
    \addvbuffer[-10pt -10pt]{
    \begin{tabular}{lccccccc}
    \toprule
    Algorithms & Msrcv2(attri) & Msrcv2(label) &Voc12(attri) & Voc12(label) & Dblp(ai) & Dblp(db) & Dblp(cv) \\
    \midrule
    cfMGML & \textbf{0.433$\pm$0.001}  & \textbf{0.382$\pm$0.018}  & \textbf{0.420$\pm$0.004}  & 0.256$\pm$0.002  & \textbf{0.576$\pm$0.005}  &\textbf{0.454$\pm$0.002} &\textbf{0.687$\pm$0.003}\\
    \textsc{Hlk} & 0.235$\pm$0.001  & 0.222$\pm$0.007  & 0.367$\pm$0.005  & \textbf{0.367$\pm$0.004}  & 0.462$\pm$0.003  &0.403$\pm$0.005 &0.644$\pm$0.006\\
    MIMLfast & 0.410$\pm$0.004  &-  & 0.382$\pm$0.008  &-  & 0.178$\pm$0.010  &0.168$\pm$0.004 &0.251$\pm$0.004\\
    \textsc{Sgsl svm} & 0.638$\pm$0.003  & 0.502$\pm$0.007  & 0.440$\pm$0.001  & 0.372$\pm$0.002  & 0.645$\pm$0.004  &0.557$\pm$0.002 &0.754$\pm$0.003\\
    \textsc{Dummy} & 0.145$\pm$0.003  & 0.145$\pm$0.003  & 0.361$\pm$0.006  & 0.361$\pm$0.006  & 0.265$\pm$0.003  &0.325$\pm$0.005 &0.354$\pm$0.004\\

    \bottomrule
    \end{tabular}
    }
    \end{center}
    \end{scriptsize}
\end{table*}

\subsection{Prediction Settings} To meet the actual situation, the graph-level classifiers can not use the bag labels unknown in the prediction phase. In other words, it has no label restrictions in the graph-level prediction. For bag-level prediction, we treat the labels of a bag as the union of predicted labels of the graphs in the bag.

\subsection{Datasets} Two different sets of datasets are used: the image and text datasets. For image datasets, we construct labeled datasets from three real image datasets: 1.) A subset of the Microsoft Research Cambridge v2 image dataset (Msrcv2) which consists of 23 classes and 591 images; 2.) PASCAL VISUAL Object Challenge 2012 Segmentation dataset (Voc12) which contains 20 classes and 1073 images; 3.) Scene dataset introduced in \cite{zhou2006:multi} which consists of 2000 natural scene images and 5 classes. The first two datasets provide a corresponding segmentation of each image into several objects and each object is assigned a label. But the last dataset provides only bag labels, which can be used for bag-level label prediction only. We cut the images into several continuous samples based on pixel-level labels and each individual sample can be converted to a graph by applying SLIC \cite{achanta2012slic}, a superpixel-based algorithm. The node of graph corresponds to each superpixel, which is described by 729-dimensional RGB-color histogram and 144-dimensional histogram of gradients (i.e. node-attributes graph) instead of the median of color histogram using in traditional \textsc{Mgml} setting (i.e. node-label graph). Because the median of color histogram is only a one-dimensional value, it can not grasp the information of local area distribution, resulting in the loss of local visual information. The new features (RGB-color histogram + HOG) can overcome this shortcoming. Each graph edge corresponds to the pair of adjacent superpixel. Hence, each image can be treated as a bag of graphs.

For text datasets, we use the DBLP dataset in three main fields: Artificial Intelligence (AI), Data Base (DB) and Computer Vision (CV). Each paper can be regarded as a union of abstracts of itself and its references which is assigned to multiple labels simultaneously. We extract several keywords and relation between these keywords from each abstract in paper by applying E-FCM algorithm \cite{perusich2006efcm}. Then each paper can be treated as a bag of graphs by converting each abstract to a graph with keywords as nodes and relation of them as edges. For AI dataset, we choose 1000 papers from top 13 frequent classes ranked in decreasing order to form the corresponding multi-graph dataset with a total of 7593 graphs. For DB dataset, there are 1050 papers and 9682 graphs totally. Moreover, we select 1300 papers from CV dataset with 10835 graphs.

\subsection{Parameter Tunning and Evaluation Measures} Since the prediction performance of the same algorithm in different datasets varies by using different parameter settings, in order to make a fair comparison, we adopt the best parameter setting for each algorithm for achieving the best performance. Specifically, we run 10-fold cross validation on each algorithm for searching optimal parameter over the set and record the average metric value over all folds $\pm$ the standard deviation in metric value. For cfMGML and \textsc{Hlk}, we search $\lambda$ over the set $\{10^{-1},...,10^{-8}\}$ and fix $R$$=$$10$, $T$$=$$100$ by observing that the subgradient descent algorithm converges in 100 iterations for most datasets. For \textsc{M3Miml}, we tune parameter $C$ over the set $\{10^{-2},...,10^{6}\}$. For the \textsc{Sgsl svm}, we use \textsc{Libsvm} with regularization parameter $C$ in $\{10^1,...,10^7\}$. For MIMLfast and MGMLent, parameters are determined in the same way as suggested by the corresponding papers. GH and WL kernels are used for node-attributes and node-label graphs respectively. \\
\indent For graph label prediction, we consider the correctly predicted label ratio as graph level accuracy. In bag level, we evaluate the performance with 6 measures commonly used in multi-label learning: one-error, hamming loss, coverage, ranking loss, average precision, macro-averaging F1.

\begin{table*}[!t]
    \centering
    \caption{Predictive performance of each algorithm (mean$\pm$std. deviation) for Bag Label prediction on the real world datasets. $\downarrow(\uparrow)$ means the smaller (larger) the value is, the better the performance is.}
    \label{Result3}
    \addvbuffer[0pt 0pt]{
    \resizebox{\textwidth}{63mm}{
    \addvbuffer[-5pt -15pt]{
    \begin{tabular}{lcccccccc}
        \hline
        \hline
        \multicolumn{1}{c}{\multirow{1}{*}{\begin{tabular}[c]{@{}c@{}}\textbf{Algorithms}\end{tabular}}} &cfMGML &cfMGML(G)   &\textsc{Hlk}     &\textsc{M3Miml}      & MIMLfast               &MGMLent               &\textsc{Sgsl svm}        &\textsc{Dummy} \\
        \hline
        \multicolumn{1}{l}{\multirow{1}{*}{\begin{tabular}[c]{@{}c@{}}\textbf{Datasets}\end{tabular}}} & \multicolumn{8}{c}{One-error$\downarrow$}\\ \cline{2-8}
        \hline
        Msrcv2        &\textbf{0.165$\pm$0.013} &0.281$\pm$0.016  &0.483$\pm$0.052 &0.289$\pm$0.012      &0.279$\pm$0.021         &\textbf{0.163$\pm$0.012}          &0.116$\pm$0.021  &0.641$\pm$0.026\\
        Scene        &\textbf{0.299$\pm$0.018} &0.312$\pm$0.020 &0.728$\pm$0.058 &0.525$\pm$0.016 	    &0.351$\pm$0.025  		 &0.334$\pm$0.023          &-  &-  \\
        Voc12         &\textbf{0.252$\pm$0.016} &0.325$\pm$0.016 &0.320$\pm$0.074 &0.447$\pm$0.013  	&0.258$\pm$0.027  		 &0.261$\pm$0.034          &0.255$\pm$0.023  &0.306$\pm$0.021\\
        Dblp(cv)               &\textbf{0.297$\pm$0.019} &0.351$\pm$0.025 &0.349$\pm$0.070 &0.661$\pm$0.007 		&0.542$\pm$0.025  		 &\textbf{0.297$\pm$0.007}          &0.223$\pm$0.024  &0.512$\pm$0.013  \\
        Dblp(ai)               &\textbf{0.370$\pm$0.020} &0.426$\pm$0.023 &0.497$\pm$0.051 &0.462$\pm$0.015      &0.536$\pm$0.020  		 &0.375$\pm$0.022          &0.302$\pm$0.030  &0.552$\pm$0.024   \\
        Dblp(db)               &\textbf{0.408$\pm$0.012} &0.439$\pm$0.029 &0.492$\pm$0.011 &0.454$\pm$0.023  	&0.545$\pm$0.014         &0.422$\pm$0.010          &0.360$\pm$0.021  &0.544$\pm$0.022 \\
        \hline
        \multicolumn{1}{l}{\multirow{1}{*}{\begin{tabular}[c]{@{}c@{}}\textbf{ }\end{tabular}}} & \multicolumn{8}{c}{Hamming loss$\downarrow$}\\ \cline{2-8}
        \hline
        Msrcv2        &\textbf{0.401$\pm$0.005} &0.511$\pm$0.015 &0.886$\pm$0.001    &0.882$\pm$0.001   &0.901$\pm$0.001        &0.452$\pm$0.003        &0.544$\pm$0.004        &1.320$\pm$0.003 \\
        Scene        &\textbf{0.565$\pm$0.005} &0.570$\pm$0.017 &0.755$\pm$0.008  	&0.891$\pm$0.002    &\textbf{0.566$\pm$0.003}        &0.581$\pm$0.006        &-        &-  \\
        Voc12         &\textbf{0.563$\pm$0.005} &0.693$\pm$0.008 &0.884$\pm$0.007    &0.658$\pm$0.004  	&0.903$\pm$0.002        &0.624$\pm$0.002 		&0.485$\pm$0.006  		&0.604$\pm$0.008  \\
        Dblp(cv)               &\textbf{0.500$\pm$0.004}  &0.574$\pm$0.003 &0.872$\pm$0.005    &0.873$\pm$0.001  	&0.689$\pm$0.002  		&0.519$\pm$0.004  		&0.355$\pm$0.003  		&0.592$\pm$0.006  \\
        Dblp(ai)               &\textbf{0.482$\pm$0.006} &0.581$\pm$0.006 &0.870$\pm$0.002  	&0.519$\pm$0.002	&0.603$\pm$0.001 		&0.517$\pm$0.002  		&0.371$\pm$0.001  	 	&0.581$\pm$0.003   \\
        Dblp(db)               &\textbf{0.645$\pm$0.002} &0.690$\pm$0.029 &0.853$\pm$0.005    &0.721$\pm$0.001 	&0.664$\pm$0.005	    &0.665$\pm$0.003  		&0.604$\pm$0.004  		&0.714$\pm$0.005   \\
        \hline
        \multicolumn{1}{l}{\multirow{1}{*}{\begin{tabular}[c]{@{}c@{}}\textbf{ }\end{tabular}}} & \multicolumn{8}{c}{Coverage$\downarrow$}\\ \cline{2-8}
        \hline
        Msrcv2        &\textbf{0.459$\pm$0.012} &0.498$\pm$0.012 &0.886$\pm$0.011    &0.915$\pm$0.012 	&0.468$\pm$0.011 	    &1.170$\pm$0.018        &0.365$\pm$0.014       &0.539$\pm$0.015  \\
        Scene        &\textbf{0.877$\pm$0.015} &1.031$\pm$0.006 &1.640$\pm$0.016  	&0.776$\pm$0.014    &0.844$\pm$0.007        &1.120$\pm$0.013        &-       &-   \\
        Voc12         &\textbf{0.698$\pm$0.019} &0.701$\pm$0.014 &0.804$\pm$0.012 	&0.811$\pm$0.013 	&\textbf{0.699$\pm$0.012}        &0.891$\pm$0.010  		&0.543$\pm$0.012       &0.702$\pm$0.013   \\
        Dblp(cv)               &\textbf{0.304$\pm$0.016} &0.409$\pm$0.007 &\textbf{0.305$\pm$0.007}    &0.935$\pm$0.015  	&0.452$\pm$0.018  		&0.590$\pm$0.019  		&0.217$\pm$0.017  	   &0.637$\pm$0.018   \\
        Dblp(ai)               &\textbf{0.309$\pm$0.005} &0.323$\pm$0.007 &0.327$\pm$0.012    &0.359$\pm$0.006   &0.324$\pm$0.014  		&0.409$\pm$0.007  		&0.223$\pm$0.007  	   &0.632$\pm$0.006   \\
        Dblp(db)               &\textbf{0.373$\pm$0.009} &0.375$\pm$0.012 &0.407$\pm$0.015    &0.432$\pm$0.009  	&0.386$\pm$0.010  		&0.612$\pm$0.014  		&0.335$\pm$0.009       &0.589$\pm$0.012 \\
        \hline
        \multicolumn{1}{l}{\multirow{1}{*}{\begin{tabular}[c]{@{}c@{}}\textbf{ }\end{tabular}}} & \multicolumn{8}{c}{Ranking loss$\downarrow$}\\ \cline{2-8}
        \hline
        Msrcv2        &\textbf{0.080$\pm$0.007} &0.108$\pm$0.005  &0.225$\pm$0.005    &0.261$\pm$0.002  	 &0.107$\pm$0.004        &0.503$\pm$0.004        &0.041$\pm$0.002         &0.230$\pm$0.005  \\
        Scene        &\textbf{0.147$\pm$0.005} &0.212$\pm$0.001 &0.353$\pm$0.003    &0.197$\pm$0.004    &0.189$\pm$0.007  		&0.452$\pm$0.004        &-  		 &-  \\
        Voc12         &\textbf{0.182$\pm$0.007}  &\textbf{0.183$\pm$0.010}  &0.208$\pm$0.005	&0.222$\pm$0.002     &\textbf{0.183$\pm$0.010}   	 &0.528$\pm$0.008        &0.121$\pm$0.003  		 &0.175$\pm$0.010  \\
        Dblp(cv)               &\textbf{0.124$\pm$0.008} &0.208$\pm$0.007 &\textbf{0.125$\pm$0.009}  	&0.273$\pm$0.001     &0.482$\pm$0.004 		&0.411$\pm$0.007  		&0.076$\pm$0.008  		 &0.403$\pm$0.006 \\
        Dblp(ai)               &\textbf{0.137$\pm$0.002} &0.171$\pm$0.004 &0.169$\pm$0.004    &0.201$\pm$0.003  	 &0.570$\pm$0.003 		&0.366$\pm$0.009  		&0.064$\pm$0.006  		 &0.424$\pm$0.005  \\
        Dblp(db)               &\textbf{0.175$\pm$0.001} &0.235$\pm$0.010 &0.201$\pm$0.002	&0.253$\pm$0.004  	 &0.572$\pm$0.002  		&0.513$\pm$0.001  		&0.147$\pm$0.002  		 &0.362$\pm$0.003  	\\
        \hline
        \multicolumn{1}{l}{\multirow{1}{*}{\begin{tabular}[c]{@{}c@{}}\textbf{ }\end{tabular}}} & \multicolumn{8}{c}{Average precision$\uparrow$}\\ \cline{2-8}
        \hline
        Msrcv2        &\textbf{0.730$\pm$0.015} &0.671$\pm$0.014  &0.501$\pm$0.015   &0.540$\pm$0.017       &0.693$\pm$0.008        &0.360$\pm$0.012          &0.860$\pm$0.007        &0.402$\pm$0.021   \\
        Scene        &\textbf{0.811$\pm$0.018} &0.754$\pm$0.015 &0.541$\pm$0.014   &0.527$\pm$0.008       &0.770$\pm$0.006	      &0.520$\pm$0.016 	        &-        &-  \\
        Voc12         &\textbf{0.612$\pm$0.014} &0.532$\pm$0.008 &0.578$\pm$0.018   &0.554$\pm$0.018       &0.605$\pm$0.013        &0.465$\pm$0.015          &0.632$\pm$0.020        &0.589$\pm$0.012  \\
        Dblp(cv)               &\textbf{0.719$\pm$0.006}  &0.655$\pm$0.007  &0.662$\pm$0.019   &0.478$\pm$0.016       &0.477$\pm$0.012        &0.487$\pm$0.019  	    &0.781$\pm$0.017        &0.531$\pm$0.016  \\
        Dblp(ai)               &\textbf{0.684$\pm$0.011}  &0.607$\pm$0.016 &0.610$\pm$0.005   &0.562$\pm$0.015       &0.429$\pm$0.015        &0.582$\pm$0.018          &0.787$\pm$0.014        &0.528$\pm$0.008   \\
        Dblp(db)               &\textbf{0.648$\pm$0.009} &0.612$\pm$0.014  &0.583$\pm$0.017   &0.556$\pm$0.010       &0.452$\pm$0.018        &0.471$\pm$0.013          &0.754$\pm$0.011        &0.514$\pm$0.009    \\
        \hline
        \multicolumn{1}{l}{\multirow{1}{*}{\begin{tabular}[c]{@{}c@{}}\textbf{ }\end{tabular}}} & \multicolumn{8}{c}{Macro-averaging F1$\uparrow$}\\ \cline{2-8}
        \hline
        Msrcv2        &\textbf{0.290$\pm$0.022} &0.271$\pm$0.016  &0.188$\pm$0.025  	  &0.186$\pm$0.021  &0.281$\pm$0.025        &0.228$\pm$0.014         &0.311$\pm$0.012          &0.110$\pm$0.010  \\
        Scene        &\textbf{0.339$\pm$0.019} &0.291$\pm$0.011 &0.301$\pm$0.014  	  &0.175$\pm$0.029 	&0.311$\pm$0.014  		&0.282$\pm$0.016  		 &-  		   &-  \\
        Voc12         &\textbf{0.236$\pm$0.007} &0.175$\pm$0.020 &0.182$\pm$0.013  	  &0.165$\pm$0.017  &\textbf{0.238$\pm$0.014}  		&0.198$\pm$0.010  		 &0.267$\pm$0.012  		   &0.172$\pm$0.026  \\
        Dblp(cv)               &\textbf{0.258$\pm$0.011} &0.202$\pm$0.014  &0.201$\pm$0.024  	  &0.146$\pm$0.020  &0.244$\pm$0.011  		&0.225$\pm$0.017  		 &0.301$\pm$0.018  	       &0.168$\pm$0.024	 \\
        Dblp(ai)               &\textbf{0.290$\pm$0.021} &0.263$\pm$0.009 &0.215$\pm$0.016	  &0.189$\pm$0.011 	&\textbf{0.291$\pm$0.017} 		&0.267$\pm$0.007  		 &0.340$\pm$0.011  		   &0.116$\pm$0.007  \\
        Dblp(db)               &\textbf{0.275$\pm$0.014} &0.244$\pm$0.018 &0.243$\pm$0.012     &0.201$\pm$0.024  &\textbf{0.278$\pm$0.008} 		&0.223$\pm$0.009 		 &0.335$\pm$0.022  		   &0.121$\pm$0.013  \\
        \hline
        \hline
    \end{tabular}
    }
    }
}
\end{table*}

\subsection{Results and Analysis}
\textbf{Performance Comparison on Graph Label Prediction.} 
In this experiment, we first study the performance difference between cfMGML and other baseline algorithms for graph-level prediction where corresponding bag labels are not available in the prediction process. Table \ref{result2} presents the accuracy results for this setting. As shown in the Table \ref{result2}, cfMGML is superior to MIMLfast $2$-$4$ percent on image datasets and $29$-$43$ percent on text datasets. This indicates that graph is better for capturing local structure than instance. Additionally, cfMGML outperforms \textsc{Hlk} which indicates that graph-level prediction can be improved by minimizing the labels ranking loss using representative graph information. Meanwhile, cfMGML outperforms other methods on all the datasets except Voc12(label). However, as mentioned above, this is not surprising, because the node-label graphs cannot adequately capture the image information. In addition, due to the \textsc{Sgsl svm} is allowed to utilize the actual label of each graph in a graph bag during training while cfMGML does not, there is no doubt that \textsc{Sgsl svm} will be better than cfMGML. To some extent, \textsc{Sgsl svm} is chosen as a ground truth case. We introduce it in the experiments only to show how far our cfMGML is from its practical upper bound on the performance. Table \ref{result2} has indicated that cfMGML has less difference with \textsc{Sgsl svm} than other algorithms.\\
\textbf{Performance Comparison on Different Types of Graph Representations.} Furthermore, for image classification, we also study the performance difference while using the node-attributes graphs (i.e., Msrcv2(attri), Voc12(attri)) and node-label graphs (i.e., Msrcv2 (label), Voc12(label)) to represent the image. From Table \ref{result2}, we observe that the accuracy of each algorithm using the node-attributes graphs is higher than using the node-label graphs. For example, the difference in accuracy of cfMGML is $5$-$16$ percent for node attributes / node label graphs datasets. As expected, this validates that the combination of HOG and color histogram features is a more effective representation that captures the color, shape and texture information of image. \\
\textbf{Performance Comparison on Bag Label Prediction.} After that, we study the bag label prediction of the proposed algorithms. Table \ref{Result3} provides the experimental result of each compared method on six different datasets. As shown in Table \ref{Result3}, it is evident that cfMGML is superior to other algorithms on most datasets in term of each evaluation measure. Specifically, cfMGML significantly outperforms MIMLfast and \textsc{M3miml}. This result suggests that multi-graph representation is better than multi-instance representation to describe the data and thus provide more profitable information for label learning. Further, we find that MGMLent performs not very well compared to our method in terms of average precision on all the datasets. This verifies that degenerating graphs to binary feature vectors may result in the loss of some useful structural information. Additionally, in bag label prediction, cfMGML obviously performs better than \textsc{Hlk} on image datasets relative to in graph label prediction. \\
\textbf{Performance Comparison on Different Formats of Bag Representations.} In this section, we conduct experiments to compare the performance difference between multi-graph representation and single-graph representation. As ever mentioned, the former captures both local and global information simultaneously while the latter exploits only global information. Since graph-level prediction is not applicable in single-graph representation, we compare two formats of datasets in the bag-level prediction. One is the multi-graph data as used in all the above experiments, and the other is single graph representation, where the number of graphs in each bag is reduced to one. For single graph represented image data, we do not cut the image into multiple samples, but directly use the SLIC algorithm to cut the image and represente it as one graph. For single graph represented text data, we combine all abstracts in the paper and use E-FCM algorithm to convert it to a graph. Both of these data are running in cfMGML on bag label prediction. As shown in Table \ref{Result3}, we observe that the performance of cfMGML using multi-graph representation (i.e. cfMGML) significantly outperforms the one using single-graph representation (i.e. cfMGML(G)). For example, the average precision of cfMGML increases up to $7$ percent in Dblp(ai) compared to the cfMGML(G). This is not surprising as the multiple graphs in the bag can properly describe the local inherent structures of the objects and the correlations between these objects, while single-graph representation ignores.


\section{Related Work}
\textbf{Multi-instance Learning.} Multi-instance learning (\textsc{Mil}) was introduced in the work of \cite{dietterich1997misl} to handle complex examples that can not be represented as a single feature vector. It was subsequently applied to problems in various areas \cite{xu2019isolation-Multi-Instance,yuguoxian-mi-7,yuguoxian-miml-2,yuguoxian-miml-3}. For effectively exploring data relationship and making full use of unlabeled images to further improve performance, graph-based learning on \textsc{Mil} has attracted more and more attention in recent years \cite{2012-graph-mi-2,2008-graph-mi-3,2006-graph-mi-4,2007-graph-mi-5}. In most of the existing graph-based multi-instance methods, the graph is constructed in sample level. Zhou et al. establish the graphs for each bag by using a distance measure. The nodes in graphs represent instances while the edges show the connections between pairs of instances and the relationship between instances can be discovered \cite{zhou2009multi-instance-non-iid}.  \\
\textbf{Multi-graph Learning.} Multi-graph learning (\textsc{Mgl}) is an generalization of Multi-Instance learning, which learn a classifier from training bags of graphs instead of instances \cite{wu2014boosting,wu2014multi-graph,pang2018semi,pang2017parallel}. Wu et al. \cite{wu2014:bag} proposed the gMGFL algorithm converting the graphs in multi-graph bags into instances with a set of feature subgraphs which can be obtained by applying a graph mining algorithm. As a result, the \textsc{Mgsl} learning is degenerated to a multi-instance learning problem which can be addressed by traditional multi-instance  classification method. Later, many works extended the standard \textsc{Mgsl} framework following this line of thought \cite{wu2016:positive}. For the same reason multi-graph multi-label learning aims to extend \textsc{Miml} to the graph field. Zhu et al. \cite{zhu2018multi} proposed a \textsc{Miml} algorithm based on entropy to solve \textsc{Mgml} through mining the informative subgraph using entropy. 

\section{Conclusion}
In this paper, we investigated a new multi-graph multi-label learning task, which aims to perform labeling at both the graph and the bag levels. This problem is significantly more challenging than traditional multi-instance learning and multi-graph learning with labels only available on the bag level. To tackle the problem, we presented a supervised learning framework, cfMGML enabling both the graph and bag levels label prediction. Meanwhile, cfMGML distinguishes itself by developing learning model directly over the graphs instead of degenerating it to multi-instance learning problem as what prior multi-graph learning approaches always do. cfMGML employs graph-kernel based scoring functions for label selection in graph and bag levels. Moreover, we discussed how the proposed thresholding rank-loss objective function and subgradient descent based graphs can be used to tackle the cumulative error issue and non-convex problem. We proved that the run time of our approach is mainly dependent on the number of training examples and the solution accuracy at the subgradient descent step. Finally, we introduced new image and text datasets for the study and the experimental results indicated that cfMGML clearly outperforms the state-of-the-art algorithms.  

\bibliography{cfMGML}

\end{document}